\documentclass[review,onefignum,onetabnum]{article}

\usepackage{amssymb,amsfonts,mathtools}
\usepackage{graphicx}
\usepackage{microtype}
\usepackage{setspace}
\usepackage{enumitem}
\usepackage[font=small,labelfont=bf]{caption}
\usepackage{booktabs} 
\usepackage{bm}
\setlist[itemize]{leftmargin=2em}
\setlist[enumerate]{leftmargin=2.5em}

\usepackage{amsmath}
\usepackage{amssymb}
\usepackage{amsfonts}
\usepackage{amsthm}
\usepackage{amsopn}
\usepackage{graphicx}
\usepackage{epstopdf}
\usepackage{algorithm}
\usepackage{algorithmic}
\usepackage{enumitem}
\usepackage{hyperref}
\usepackage[nameinlink,capitalise,noabbrev]{cleveref}
\usepackage{aliascnt}

\ifpdf
  \DeclareGraphicsExtensions{.eps,.pdf,.png,.jpg}
\else
  \DeclareGraphicsExtensions{.eps}
\fi

\setlist[enumerate]{leftmargin=.5in}
\setlist[itemize]{leftmargin=.5in}

\newtheorem{theorem}{Theorem}[section]

\newaliascnt{lemma}{theorem}
\newtheorem{lemma}[lemma]{Lemma}
\aliascntresetthe{lemma}

\newaliascnt{proposition}{theorem}
\newtheorem{proposition}[proposition]{Proposition}
\aliascntresetthe{proposition}

\newaliascnt{corollary}{theorem}
\newtheorem{corollary}[corollary]{Corollary}
\aliascntresetthe{corollary}

\newaliascnt{claim}{theorem}

\aliascntresetthe{claim}

\theoremstyle{definition}

\newaliascnt{definition}{theorem}
\newtheorem{definition}[definition]{Definition}
\aliascntresetthe{definition}

\newaliascnt{assumption}{theorem}
\newtheorem{assumption}[assumption]{Assumption}
\aliascntresetthe{assumption}

\newtheorem{example}{Example}[section]

\newaliascnt{hypothesis}{theorem}

\aliascntresetthe{hypothesis}

\theoremstyle{remark}

\newaliascnt{remark}{theorem}

\aliascntresetthe{remark}

\newaliascnt{fact}{theorem}

\aliascntresetthe{fact}

\crefname{section}{Section}{Sections}
\crefname{subsection}{Subsection}{Subsections}
\crefname{definition}{Definition}{Definitions}
\crefname{theorem}{Theorem}{Theorems}
\crefname{lemma}{Lemma}{Lemmas}
\crefname{proposition}{Proposition}{Propositions}
\crefname{corollary}{Corollary}{Corollaries}
\crefname{claim}{Claim}{Claims}
\crefname{assumption}{Assumption}{Assumptions}
\crefname{remark}{Remark}{Remarks}
\crefname{example}{Example}{Examples}
\crefname{hypothesis}{Hypothesis}{Hypotheses}
\crefname{fact}{Fact}{Facts}
\crefname{figure}{Figure}{Figures}
\crefname{table}{Table}{Tables}
\crefname{equation}{Equation}{Equations}

\title{On the Regularity and Generalization of One-Step Wasserstein-guided Generative Models for PDE-Induced Measures}

\author{
Likun Lin\thanks{Department of Mathematics, The University of Hong Kong, Pokfulam Road, Hong Kong SAR, P.R. China.
\texttt{likunlin3@connect.hku.hk}.}
\and
Zhongjian Wang\thanks{Division of Mathematical Sciences, School of Physical and Mathematical Sciences, Nanyang Technological University, Singapore 637371.
\texttt{zhongjian.wang@ntu.edu.sg}.}
\and
Jack Xin\thanks{Department of Mathematics, University of California at Irvine, Irvine, CA 92697, USA.
\texttt{jxin@math.uci.edu}.}
\and
Zhiwen Zhang\thanks{Corresponding author. Department of Mathematics, The University of Hong Kong, Pokfulam Road, Hong Kong SAR, P.R. China. Materials Innovation Institute for Life Sciences and Energy (MILES), HKU-SIRI, Shenzhen, 518045, P.R. China.
\texttt{zhangzw@hku.hk}.}
}

\begin{document}
\newcommand{\Hess}{\nabla^2}
\newcommand{\tr}{\operatorname{tr}}
\newcommand{\opnorm}[1]{\left\|#1\right\|_{\mathrm{op}}}
\newcommand{\lammax}{\lambda_{\max}}
\newcommand{\Prob}{\mathcal{P}}

\maketitle

\begin{abstract}
Despite the remarkable empirical success of generative models, the available theory on their statistical accuracy in scientific computing remains largely pessimistic. This paper develops a theoretical framework for understanding the regularity of transport maps and the generalization properties of one-step Wasserstein-guided generative models for PDE-induced probability measures. We consider normalized target densities associated with linear elliptic and parabolic equations on bounded domains, as well as diffusion and Fokker--Planck equations on the torus. Under standard structural assumptions, we prove that these target measures satisfy doubling conditions. By combining this fact with regularity theory for optimal transport between doubling measures, we show that the optimal transport map from a uniform source measure to the target measure is Hölder continuous. This regularity yields an approximation-theoretic justification for one-step generative models that learn PDE-induced distributions via a single pushforward map. As a representative instance, we study DeepParticle and derive excess-risk bounds characterizing the discrepancy between the learned map and the population-optimal map. We also establish a robustness estimate under target shift and illustrate the theory with experiments which support the derived rates.
\end{abstract}


\medskip


\section{Introduction}

Deep neural networks as direct solvers or surrogate models have become an important computational paradigm across science and engineering. Representative examples include the Deep Ritz method for variational PDEs and eigenvalue problems \cite{weinan2018deepritz,wang2020meshfree,cui2023variational,chen2023solving,huang2025neural,lin2026biomimetic}, the deep Galerkin method for high-dimensional PDEs \cite{DGM}, and physics-informed neural networks (PINNs), which incorporate the governing equations directly into the loss function \cite{raissi2019pinn}. For parametric PDEs, DeepONet learns nonlinear solution operators from data \cite{lu2021deeponet}, while the Fourier Neural Operator learns mappings between function spaces and has proved effective for PDE families \cite{li2021fno,cheng2025podno,li2024pino}. These developments illustrate the rapidly growing interface between deep learning and PDE computation. 

The above paradigms are based on the Eulerian viewpoint of PDEs. By contrast, the Lagrangian approach \cite{wang2018computing,wang2021sharp,lyu2022kppfronts,zhang2025randomkpp,wu2025computing} has shown strong adaptivity in the small-diffusion regime. Under this viewpoint, deep generative models can be used to learn the map that pushes forward a simple reference distribution, such as a Gaussian or uniform measure, to a target distribution induced by the PDE solution or by the invariant measure of an interacting particle system. Representative examples include normalizing flows (NF) \cite{kobyzev2020flows}, generative adversarial networks (GANs) \cite{goodfellow2014gan}, diffusion models \cite{ho2020ddpm}, and DeepParticle (DP) \cite{wang2022deepparticle,wang2024deepparticleks}. Among them, the DP method learns target distributions by minimizing the Wasserstein distance using data generated from Lagrangian particle approaches with ample numerical  support, such as KPP front speeds 
(principal eigenvalue problem of reaction-diffusion-advection operators) and associated invariant-measure computations \cite{wang2022deepparticle}, and aggregation/focusing phenomena in Keller-Segel systems \cite{wang2024deepparticleks}. Recently, one-step diffusion models have been compared \cite{zhang2025bidirectionaldeepparticle} and integrated with DP \cite{shen2026twostepdiffusion} to achieve fast and generalizable predictions for solutions with large spatial gradients in the singular perturbation regime.

\medskip
Despite their different training mechanisms, these one-step models admit a common interpretation: they learn a parametric pushforward map transporting a simple reference measure to a target distribution while minimizing a Wasserstein-type discrepancy between the generated and target laws. In DeepParticle, the minimization objective is formulated in the primal Wasserstein form, while in WGAN-type models \cite{arjovsky2017wgan}, the Wasserstein distance is optimized through its dual formulation. From this perspective, the theory of one-step Wasserstein-guided generative modeling centers on two questions: whether the population transport map is regular enough to be represented and approximated by neural networks, and whether a map learned from empirical data generalizes to the population level.

The regularity issue is crucial because finite feedforward networks can represent only continuous functions, whereas pushforward maps need not be continuous. As noted in \cite{makkuva2020icnntransport}, approximating a discontinuous transport map by a continuous network can create spurious mass rather than correctly transporting the reference law to the target measure. Moreover, since the learned map \(\mathcal{T}_\theta\) is obtained from sampled measures, its population-level performance must also be controlled. A natural criterion is the population excess risk \(R(\mathcal{T}_\theta;\mu,\nu)-R(\mathcal{T};\mu,\nu)\), defined later in \eqref{eq:population-excess-risk}, where \(\mathcal{T}\) denotes the population-optimal transport map \cite{mohri2012foundations}; this quantity captures both the generalization and optimality gap between the learned map and the population-optimal one. 

Motivated by out-of-distribution (OOD) generalization, we further consider a
target-shift setting in which a map learned from the source PDE-induced target
law \(\nu\) is evaluated under an unseen target law \(\nu_1\), while the
reference law \(\mu\) remains fixed. In contrast to the population excess risk
for the source target distribution, this criterion measures robustness under target shift
through the OOD gap
\(R(\mathcal{T}_\theta;\mu,\nu_1)-R(\mathcal{T};\mu,\nu)\). This setting is natural in scientific computing, where
perturbing physical parameters, such as the diffusion coefficient, can change
the PDE-induced target distribution while the generator reference law remains
fixed.

The existing theory still does not resolve these two issues simultaneously, especially in the PDE-induced background. For multi-step generative models, such as diffusion models \cite{ho2020ddpm}, sampling is defined through a stochastic evolution rather than a single explicit pushforward map, which makes both direct approximation analysis and population-level generalization analysis substantially more difficult and indirect \cite{wang2024wasserstein,de2022convergence}. For one-step generative models, existing guarantees mainly fall into the statistical convergence under H\"older-type integral probability metrics (IPMs) \cite{chakraborty2025generative,stephanovitch2024wassersteingan}, results formulated directly in Wasserstein-1 distance \cite{chen2022distributiongan,liang2021gan}, or neural-IPM-based analyses of Wasserstein GANs \cite{biau2021wgan}. While these works provide important statistical guarantees, they are mostly developed under Wasserstein-1 or related IPMs, rather than the Wasserstein-2 criterion considered here, or rely on abstract generator--discriminator complexity arguments. As a result, they do not directly provide a practical framework that simultaneously explains why the relevant transport map should be learnable and why the learned map should generalize well at the population level. By contrast, PDE-induced measures carry additional structure inherited from the underlying equations, which makes both the regularity of the associated transport maps and the resulting excess-risk behavior more explicit and verifiable.

\medskip
This paper develops a theoretical framework for one-step Wasserstein-guided generative modeling with PDE-induced target measures. We first prove the existence of a regular optimal transport map at the population level, providing a rigorous basis for learning the underlying transport. In contrast to regularity results that typically require either an unbounded domain \cite{meng2025pathway,kim2012caffarelli} or target densities bounded away from zero \cite{caffarelli1992regularity}, our analysis relies on a doubling measure condition \cite{jhaveri2022regularity}, which allows the target density to be degenerate on bounded domains while still yielding H\"older regularity of the optimal transport map. 

To make this result directly verifiable in concrete PDE settings, we formulate explicit assumptions on the coefficients, boundary data, and initial conditions rather than postulating regularity of the solution a priori. This regularity also enables approximation of the transport map by neural networks, including sigmoid networks and convolutional neural networks \cite{yang2025cnn,langer2021sigmoid}. For concreteness, and in view of their widespread use in practical generative models such as WGAN-type generators, we focus on ReLU networks in the subsequent analysis. The approximation theory for ReLU networks developed in \cite{shen2020neurons} implies that, when the network class is sufficiently expressive, the resulting approximation error can be made sufficiently small.

Furthermore, we study DeepParticle as a representative one-step Wasserstein-guided generative model and carry out an excess-risk analysis that links the population problem to its empirical implementation. By decomposing the excess risk into generalization, discretization, approximation, and optimization errors, we derive an excess-risk bound and corresponding convergence rates in Wasserstein-2 distance. In particular, for sufficiently expressive ReLU networks with sufficiently small optimization error, the expected excess risk decays with the sample size \(N\) at the rates \(N^{-1/4}\) for \(d<4\), \(N^{-1/4}\sqrt{\log(1+N)}\) for \(d=4\), and \(N^{-1/d}\) for \(d>4\), where \(d\) is the spatial dimension. This predicted convergence rate is also confirmed by numerical experiments, for which comparable validations remain limited in the literature on one-step Wasserstein-guided generative models.

The rest of the paper is organized as follows. In \cref{sec:preliminaries}, we introduce the transport-learning formulation and then specialize it to DeepParticle. We also review the PDE settings covering bounded-domain elliptic and parabolic problems, as well as Fokker--Planck equations on the torus, and formulate them via easily verifiable assumptions. In \cref{sec:main-results}, we present our main regularity results for optimal transport maps associated with these PDE-induced target measures. In \cref{sec:proof-main}, we prove these regularity results in the elliptic, parabolic, and torus settings and establish the corresponding generalization properties. Based on the regularity theory, we turn to DP and derive an excess-risk bound. Then, in \cref{sec:numerical}, we present numerical experiments that verify the predicted convergence behavior. Finally, we conclude in \cref{sec:conclusion}.

\section{Preliminaries}\label{sec:preliminaries}
In this section, we introduce the transport-learning formulation used throughout the paper and the PDE settings that generate the target measures of interest. We first describe the population and empirical Wasserstein risks for one-step transport learning and then specialize the discussion to DeepParticle. We next review the elliptic, parabolic, and torus models from which the target measures are constructed.

\subsection{Transport-learning formulation and DeepParticle}

Let $X,Y\subset \mathbb{R}^d$ be measurable sets, and let $\mu\in\mathcal{P}_2(X)$ and $\nu\in\mathcal{P}_2(Y)$ denote the \emph{source} and \emph{target} probability measures, respectively. For any measurable map $S\colon X\to Y$, we write $S_{\#}\mu$ for the pushforward of $\mu$, defined by $(S_{\#}\mu)(A)=\mu\bigl(S^{-1}(A)\bigr)$ for every measurable $A\subset Y$. In one-step generative modeling, the goal is to learn a single map that pushes forward a simple reference measure $\mu$ to the target measure $\nu$. In the PDE-induced settings studied below, $\mu$ will typically be a uniform measure on a bounded reference domain, while $\nu$ will be obtained by normalizing a PDE solution.

To avoid notational ambiguity later, we reserve $\mathcal{T}$ for the population optimal transport map when it exists, and use $S$ or $\mathcal{T}_\theta$ for learned maps. For any candidate map $S$, we define the population risk by 
\begin{equation}\label{eq:R}
R(S;\mu,\nu) \triangleq W_2\!\left(S_{\#}\mu,\nu\right).
\end{equation}
When only samples are available, let $\{x_i\}_{i=1}^N\sim\mu$ and $\{y_i\}_{i=1}^N\sim\nu$ be i.i.d. samples, and form the empirical measures $\hat{\mu}_N \triangleq \frac{1}{N}\sum_{i=1}^N \delta_{x_i}$ and $\hat{\nu}_N \triangleq \frac{1}{N}\sum_{i=1}^N \delta_{y_i}$. The empirical risk is then $R(S;\hat{\mu}_N,\hat{\nu}_N) \triangleq W_2\!\left(S_{\#}\hat{\mu}_N,\hat{\nu}_N\right)$, which is evaluated either by linear
programming or approximately by the iterative divide-and-conquer algorithm \cite{wang2022deepparticle}. Given a neural-network hypothesis class $\mathcal{H}$, DeepParticle returns a map $\mathcal{T}_\theta\in\mathcal{H}$ that approximately minimizes the empirical risk,
\begin{equation*}
    \mathcal{T}_\theta \approx \arg\min_{S\in\mathcal{H}} R(S;\hat{\mu}_N,\hat{\nu}_N).
\end{equation*}
Thus, DeepParticle is a representative one-step Wasserstein-guided generative model, which proposes to compute $W_2$ in the primal space.


\subsection{Bounded-domain elliptic and parabolic settings}
We now introduce the bounded-domain PDE settings that may arise from various applications of scientific computing and serve as the target measures. We first fix the elliptic operator shared by the elliptic and parabolic problems and then discuss the corresponding normalization procedures.

Let $Y\subset\mathbb{R}^d$ be a bounded $C^{2,\alpha}$ domain for some $\alpha\in(0,1)$, and let $a^{ij}, b^i, c\in C^{0,\alpha}(\overline{Y})$ with $a^{ij}=a^{ji}$. We consider the second-order operator in nondivergence form
\begin{equation}\label{eq:operator}
\mathcal{L}u(y)
= -\sum_{i,j=1}^d a^{ij}(y)\,\partial_{ij}u(y)
  + \sum_{i=1}^d b^i(y)\,\partial_i u(y)
  + c(y)u(y).
\end{equation}
We assume uniform ellipticity for this operator: there exist constants $0<\lambda\le \Lambda$ such that
\[
\lambda |\xi|^2 \le \sum_{i,j=1}^d a^{ij}(y)\xi_i\xi_j \le \Lambda |\xi|^2,
\qquad \forall\,\xi\in\mathbb{R}^d,\ \forall\,y\in\overline{Y}.
\]
Throughout the bounded-domain discussion, we also write $d(y):=\operatorname{dist}(y,\partial Y)$.

These are standard assumptions in elliptic and parabolic Schauder theory and will be used repeatedly below.

We first consider the second-order linear elliptic PDE on the bounded domain. 
Given forcing $g$ in $Y$ and boundary data $h$ on $\partial Y$, consider the Dirichlet problem
\begin{equation}\label{eq:elliptic_pde}
\begin{cases}
\mathcal{L}u = g & \text{in } Y,\\
u = h & \text{on } \partial Y.
\end{cases}
\end{equation}

\begin{assumption}\label{ass:elliptic_data}
Assume $g\in C^{\alpha}(\overline{Y})$, $h\in C^{2,\alpha}(\partial Y)$, $c\ge 0$ in $\overline{Y}$, $g\ge 0$ in $\overline{Y}$, and $h\geq 0$ on $\partial Y$. In addition, $g$ and $h$ are not both identically zero. We will distinguish the cases $h\equiv 0$ on $\partial Y$ or $h>0$ on $\partial Y$.

\end{assumption}

\begin{proposition}[Existence and nonnegativity of the solution]\label{prop:elliptic_reg_pos}
Under \cref{ass:elliptic_data}, problem~\eqref{eq:elliptic_pde} admits a solution $u\in C^{2,\alpha}(\overline{Y})$. In particular, $u\in C^2(Y)\cap C(\overline{Y})$, and $u\ge 0$ on $\overline{Y}$.
\end{proposition}

\begin{proof}
Classical Schauder theory yields $u\in C^{2,\alpha}(\overline{Y})$; see, for example, \cite{gilbarg1983elliptic}. Since $c\ge 0$, $g\ge 0$, and $h\ge 0$, the maximum principle implies $u\ge 0$ on $\overline{Y}$.

Moreover, $u\not\equiv 0$. Indeed, if $u\equiv 0$, then \eqref{eq:elliptic_pde} would force $g\equiv 0$ in $Y$ and $h\equiv 0$ on $\partial Y$, contradicting \cref{ass:elliptic_data}. In particular, when $h\equiv 0$ on $\partial Y$, one has $g\not\equiv 0$, and the strong maximum principle yields 
$
u>0 \text{ in } Y.
$
\end{proof}

By \cref{prop:elliptic_reg_pos}, we have $u\ge 0$ on $\overline{Y}$. Since $u\not\equiv 0$, it follows that $\int_Y u(y)\,dy>0$; since $Y$ is bounded and $u\in C(\overline{Y})$, we also have $\int_Y u(y)\,dy<\infty$. We then define the normalized density:
 $$\tilde{u}(y) := \frac{u(y)}{\int_Y u(z)\,dz}$$
 In the elliptic setting, the target measure $\nu$ is the probability measure with density $\tilde{u}$.

Having introduced the elliptic setting, we now turn to its time-dependent counterpart. In the parabolic case, the target measure is obtained from the solution at a fixed terminal time $T>0$.

Let $Q_T := Y\times(0,T)$ and $\Gamma_T := (\partial Y\times[0,T])\cup(\overline{Y}\times\{0\})$. Consider the parabolic problem
\begin{equation}\label{eq:parabolic_pde}
\begin{cases}
\partial_t u + \mathcal{L}u = g & \text{in } Q_T,\\
u = h & \text{on } \Gamma_T,
\end{cases}
\end{equation}
where $\mathcal{L}$ is as in \cref{eq:operator}.

\begin{assumption}\label{ass:parabolic_data}
Assume $g\in C^{\alpha,\alpha/2}(\overline{Q_T})$, $h\in C^{2+\alpha,1+\alpha/2}(\Gamma_T)$, and the standard compatibility conditions on $\Gamma_T$. Moreover, $c\ge 0$ in $\overline Y$, $g\ge 0$ in $\overline{Q_T}$, and $h\ge 0$ on $\Gamma_T$, and $g$ and $h$ are not both identically zero. We will distinguish the cases $u(\cdot,T)=0$ on $\partial Y$ and $u(\cdot,T)>0$ on $\partial Y$.
\end{assumption}

\begin{proposition}[Regularity and nonnegativity]\label{prop:parabolic_reg_pos}
Under \cref{ass:parabolic_data}, problem~\eqref{eq:parabolic_pde} admits a solution $u\in C^{2+\alpha,1+\alpha/2}(\overline{Q_T})$. In particular, $u\in C^{2,1}(Q_T)\cap C(\overline{Q_T})$, and $u\ge 0$ on $\overline{Q_T}$.
\end{proposition}

\begin{proof}
Classical parabolic theory yields $u\in C^{2+\alpha,1+\alpha/2}(\overline{Q_T})$; see, for example, \cite{lieberman1996parabolic}. Since $g\ge 0$ in $Q_T$ and $h\ge 0$ on $\Gamma_T$, the parabolic maximum principle implies $u\ge 0$ on $\overline{Q_T}$.
\end{proof}

By \cref{prop:parabolic_reg_pos}, we have $u\ge 0$ on $\overline{Q_T}$. Moreover, $u$ is not identically zero: otherwise $u\equiv 0$ would force $g\equiv 0$ on $Q_T$ and $h\equiv 0$ on $\Gamma_T$, contrary to \cref{ass:parabolic_data}, so it follows that $\int_Y u(y,T)\,dy>0$. Since $Y$ is bounded and $u\in C(\overline{Y})$, we also have $\int_Y u(y,T)\,dy<\infty.$
Hence we define the terminal-time density:
$$
\tilde{u}(y,T) := \frac{u(y,T)}{\int_Y u(z,T)\,dz}.
$$
In the parabolic setting, the target measure $\nu$ is the probability measure with density $\tilde{u}(\cdot,T)$.

\subsection{Fokker--Planck on the torus}\label{sec:fp_torus}

We next consider diffusion on the torus, which provides a periodic setting without boundary effects. 

Let \(\mathbb{T}^d:=\mathbb{R}^d/(L\mathbb{Z})^d\), identified with the
periodic cell \([0,L]^d\), and equip it with normalized Lebesgue measure.
Let \(\bm{b}\in C^{\infty}(\mathbb{T}^d;\mathbb{R}^d)\) and
\(\kappa\in C^{\infty}(\mathbb{T}^d)\) be space-periodic coefficients, and
assume that there exist constants \(0<\kappa_-\le \kappa_+<\infty\) such that $ \kappa_- \le \kappa(x)\le \kappa_+$. Consider the diffusion process \((X_t)_{t\ge0}\) on
\(\mathbb{T}^d\) defined by
\begin{equation}\label{eq:sde_torus}
dX_t = \bm{b}(X_t)\,dt + \sqrt{2\kappa(X_t)}\,dW_t,
\qquad X_0\sim \rho_0,
\end{equation}
where the initial density \(\rho_0\in L^1(\mathbb{T}^d)\) satisfies \(\rho_0\ge 0\), and $\int_{\mathbb{T}^d}\rho_0(x)\,dx = 1.$

Let \(u(x,t)\) denote the probability density of \(X_t\) with respect to
normalized Lebesgue measure. Then \(u\) satisfies the Fokker--Planck equation
\begin{equation}\label{eq:FP}
\begin{cases}
\partial_t u = \mathcal{L}^*u, & (x,t)\in\mathbb{T}^d\times(0,\infty),\\
u(x,0)=\rho_0(x), & x\in\mathbb{T}^d,
\end{cases}
\end{equation}
where the formal adjoint is $\mathcal{L}^*u=\Delta\bigl(\kappa(x)u\bigr)-\nabla\cdot\bigl(\bm{b}(x)u\bigr).$

\begin{proposition}[Smoothness, mass conservation, and positivity on the torus]\label{prop:fp_smooth}
The solution \(u\) of \eqref{eq:FP} belongs to
\(C^{\infty}\bigl((0,\infty)\times\mathbb{T}^d\bigr)\), satisfies
\(u(t,\cdot)\ge 0\), and obeys $\int_{\mathbb{T}^d}u(t,x)\,dx = 1$ for all \(t\ge 0\).
\end{proposition}

\begin{proof}
Since \eqref{eq:FP} is uniformly parabolic with smooth coefficients, standard
parabolic regularity implies that \(u\in C^\infty((0,\infty)\times\mathbb T^d)\). Moreover, since \(u(t,\cdot)\) is the probability density of \(X_t\), it is
nonnegative and has total mass one.
\end{proof}

Since the coefficients are time-independent, the long-time state is a
time-independent invariant density. Namely, there exists a unique smooth positive
density \(m(x)\) such that the solution of the Fokker--Planck equation converges to this stationary state as
\begin{equation}\label{eq:stationary-invariant-density}
0=\Delta(\kappa m)-\nabla\cdot(\bm{b}m), \qquad \displaystyle\int_{\mathbb{T}^d}m(x)\,dx=1.
\end{equation}
Equivalently, \(m(x)\,dx\) is the invariant probability measure for the diffusion process on the torus.
\paragraph{Application: KPP front speed and a tilted diffusion}\label{sec:kpp_application}
The KPP front-speed experiments in \cite{wang2022deepparticle} provide a concrete example of the preceding torus diffusion/Fokker--Planck framework. We now use the theory above to explain why the corresponding DeepParticle computation is feasible. Consider the KPP equation
\[
\partial_t q=\kappa\Delta_{\bm{x}}q+(\bm{v}(\bm{x})\cdot\nabla_{\bm{x}})q+\tau^{-1}f(q),
\qquad \bm{x}\in\mathbb{T}^d,
\]
where $\kappa>0$ is the diffusion coefficient, $\tau>0$ is the reaction time scale, $f(q)=q(1-q)$ is the KPP nonlinearity, and $\bm{v}(\bm{x})$ is a space-periodic, mean-zero, divergence-free velocity field. Following the interacting-particle formulation in \cite{lyu2022kppfronts}, the minimal front speed in direction $\bm{e}$ is
\[
c^*(\bm{e})=\inf_{\alpha>0}\frac{\lambda(\alpha)}{\alpha},
\]
where $\lambda(\alpha)$ is the principal eigenvalue of the parabolic operator $\partial_t-\mathcal{A}$ with
\[
\mathcal{A}w
:=
\kappa\Delta_{\bm{x}}w
+
(2\alpha\bm{e}+\bm{v})\cdot\nabla_{\bm{x}}w
+
V(\bm{x})w,
\]
and
\[
V(\bm{x})
:=
\kappa\alpha^2
+
\alpha\,\bm{v}(\bm{x})\cdot\bm{e}
+
\tau^{-1}f'(0).
\]
The Feynman--Kac formula gives
\[
\lambda(\alpha)
=
\lim_{t\to\infty}\frac1t
\log
\mathbb{E}\!\left[
\exp\!\left(
\int_0^t V(X_s^{t,\bm{x}})\,ds
\right)
\right],
\]
where $X_s^{t,\bm{x}}$ solves
\[
dX_s^{t,\bm{x}}
=
\bm{b}_{\alpha}(X_s^{t,\bm{x}})\,ds
+
\sqrt{2\kappa}\,dW_s,
\qquad
X_0^{t,\bm{x}}=\bm{x},
\]
with drift
\[
\bm{b}_{\alpha}:=2\alpha\bm{e}+\bm{v}.
\]
Thus the tilted KPP diffusion is a special case of the Fokker--Planck equation considered above. Since $\bm{v}$ is space periodic, the relevant invariant object is the invariant density associated with the drift $\bm{b}_{\alpha}$, or equivalently the invariant measure of the augmented diffusion on the time torus.

\section{Main Results}\label{sec:main-results}

In this section, we present the two main results of the paper.  
The first is a \emph{regularity result}: for the PDE-induced target measures introduced in \cref{sec:preliminaries}, the target measure is doubling, and consequently the associated optimal transport map is H\"older continuous.  
The second is an \emph{excess-risk result}: once the population-optimal transport map enjoys this regularity, one obtains a quantitative excess-risk bound for the DeepParticle method based on approximation theory and empirical Wasserstein estimates.

\subsection{Regularity of transport map}

We begin with the regularity result that underlies the later learning analysis. The key point is that the PDE-induced target measures introduced in \cref{sec:preliminaries} become doubling after normalization and, in the torus case, after passing to a periodic representative on a fundamental cube. H\"older regularity of the optimal transport map then follows from the general theory of optimal transport between doubling measures.

\begin{theorem}[Regularity of optimal transport for PDE-induced target measures]\label{thm:main1}
Let \(X\text{ and }Y\subset\mathbb R^d\) be a bounded open convex set, and let \(\mu\) be the uniform probability measure on \(X\) and \(\nu\) be a target measure obtained from the PDE \eqref{eq:elliptic_pde}, \eqref{eq:parabolic_pde}, the Fokker--Planck equation \eqref{eq:FP}, or the invariant-density problem \eqref{eq:stationary-invariant-density}. Then \(\nu\) is a doubling measure on its support. Consequently, if \(\mathcal{T}\) denotes the optimal transport map from \(\mu\) to \(\nu\), then $\mathcal{T}\in C^\beta(\overline X)$ for some \(\beta\in(0,1]\). The exponent \(\beta\) depends only on the dimension, the doubling constants of \(\mu\) and \(\nu\), and the geometry of the source and target domains.
\end{theorem}

\subsection{Excess risk upper bound}

Throughout this section, we call \(\mathcal{T}\) \emph{regular} if \(\mathcal{T}\in C^\alpha(\overline X)\) for some \(\alpha\in(0,1]\). The regularity result above enters here only through this property. Once regularity is available, the excess-risk estimate follows from approximation theory together with empirical Wasserstein bounds.

For any candidate map \(S\), we define its population excess risk as
\begin{equation}
R(S;\mu,\nu)-R(\mathcal T;\mu,\nu),
\label{eq:population-excess-risk}
\end{equation}
where \(R\) is defined in \cref{eq:R}. This quantity measures the additional population risk incurred by using \(S\) instead of the target map \(\mathcal T\). In particular, a small excess risk means that \(S\) performs nearly as well as \(\mathcal T\) at the population level.

\begin{theorem}[Excess-risk bound under regularity of \(\mathcal{T}\)]
\label{thm:main-excess-risk}
Let \(X,Y\subseteq[-B,B]^d\) be a bounded open set, let \(\mu\) be the uniform probability measure on \(X\), and let \(\nu\in\mathcal P_2(Y)\) be a target measure. Suppose that the optimal transport map \(\mathcal{T}\) from \(\mu\) to \(\nu\) is regular. Fix \(W,L\in\mathbb N^+\), independently of \(N\), and let \(\mathcal H\) be a hypothesis class that contains the ReLU architectures with width \(3^{d+3}\max\{d\lfloor W^{1/d}\rfloor,\,W+1\}\) and depth \(12L+14+2d\). Assume that \(\operatorname{Lip}(S)\le L_{\mathcal H}\) for all \(S\in\mathcal H\), and that \(\varepsilon_{\mathrm{opt},N}\in L^1\), where \(\varepsilon_{\mathrm{opt},N}\coloneqq R(\mathcal{T}_{\theta,N};\hat\mu_N,\hat\nu_N)-R(\mathcal{T}_N^\ast;\hat\mu_N,\hat\nu_N)\). Then
\[
\mathbb E\!\left[R(\mathcal{T}_{\theta,N};\mu,\nu)-R(\mathcal{T};\mu,\nu)\right]
\le
\mathrm{Stat}_N(\mathcal H)+\mathrm{App}(\mathcal H;\mathcal{T})+\mathbb E[\varepsilon_{\mathrm{opt},N}].
\]
where we write
\(\mathrm{Stat}_N(\mathcal H)\) for the statistical error arising from
the empirical measure, and \(\mathrm{App}(\mathcal H;\mathcal T)\) for the approximation error, which measures how well the class \(\mathcal H\) can approximate the regular
transport map \(\mathcal T\):
\[
\mathrm{Stat}_N(\mathcal H)
\coloneqq
L_{\mathcal H}\,\mathbb E[W_2(\mu,\hat\mu_N)]
+
3\,\mathbb E[W_2(\nu,\hat\nu_N)], \quad \mathrm{App}(\mathcal H;\mathcal T)
\coloneqq
C_{\alpha,d,B,\mathcal T}\,W^{-2\alpha/d}L^{-2\alpha/d}.
\]
\end{theorem}

For each PDE-induced target measure $\nu$ covered by \cref{thm:main1}, the
corresponding optimal transport map $\mathcal{T}$ is regular. Hence
\cref{thm:main-excess-risk} applies in the elliptic, parabolic, and torus
settings introduced in \cref{sec:preliminaries}.

We then highlight a consequence of this result: once the architecture has been
chosen sufficiently large, the remaining dependence on the sample size $N$ is
purely statistical.
\begin{corollary}[Fixed architecture and convergence in $N$]
\label{cor:fixed-arch-rate}
Under the assumptions of \cref{thm:main-excess-risk}, let $\eta>0$. Then one can choose a hypothesis class $\mathcal H_\eta$ of the form specified above, with width and depth independent of $N$, such that $\mathrm{App}(\mathcal H_\eta;\mathcal{T})\le \eta$. If the corresponding trained maps $\mathcal{T}_{\theta,N}\in\mathcal H_\eta$ satisfy $\mathbb E[\varepsilon_{\mathrm{opt},N}]\le \eta$ for all $N$, then
\[
\mathbb E\!\left[R(\mathcal{T}_{\theta,N};\mu,\nu)-R(\mathcal{T};\mu,\nu)\right]
\le
\mathrm{Stat}_N(\mathcal H_\eta)+2\eta.
\]
In particular, there exists a constant $C_{\eta}>0$, independent of $N$, such that
\[
\mathbb E\!\left[R(\mathcal{T}_{\theta,N};\mu,\nu)-R(\mathcal{T};\mu,\nu)\right]
\le
2\eta+C_{\eta}
\begin{cases}
N^{-1/4}, & d<4,\\
N^{-1/4}\sqrt{\log(1+N)}, & d=4,\\
N^{-1/d}, & d>4.
\end{cases}
\]
\end{corollary}

Thus, for the PDE-induced target measures considered above, one may first
choose a sufficiently expressive architecture, independently of $N$, so that
the approximation error is below a prescribed tolerance $\eta$. After this
architecture has been fixed, the convergence in $N$ is governed entirely by
the statistical term, provided that the optimization error is controlled at
the same order.

\section{Proof of  Main Results}\label{sec:proof-main}
\subsection{Regularity of transport maps} 

The proof of the regularity of the transport map has two ingredients. First, we verify that the target measure is doubling in each PDE setting. Second, we invoke the regularity theorem for optimal transport maps between doubling measures. In the elliptic and parabolic cases, convexity of the spatial domain is used through the concavity of the distance-to-the-boundary function.

We first recall the geometric form of the doubling condition that will be 
used throughout. The relevant test sets are ellipsoids contained in the 
target domain.

\begin{definition}[Ellipsoid]
Let \(E\) be a symmetric positive definite matrix and let \(x\in\mathbb R^d\). 
The ellipsoid generated by \(E\) and centered at \(x\) is
\[
\mathcal E_{E,x}:=x+E(B_1(0)).
\]
For \(r>0\), we write \(\mathcal E_{rE,x}\) for the dilation of 
\(\mathcal E_{E,x}\) with respect to its center.
\end{definition}

\begin{definition}[Doubling measure, {\normalfont\cite{jhaveri2022regularity}}]\label{def:doubling}
Let \(Y\subset\mathbb R^d\) be a bounded convex domain. A finite non-negative 
measure \(\eta\) on \(Y\) is called \emph{doubling} if there exists 
\(C_{\mathrm{dbl}}>0\) such that, for every ellipsoid 
\(\mathcal E_{E,x}\subset Y\) whose center \(x\) lies in \(\operatorname{spt}\eta\),
\[
\eta(\mathcal E_{E,x})
\le 
C_{\mathrm{dbl}}\,
\eta\!\left(\mathcal E_{\frac12E,x}\right).
\]
\end{definition}

The next lemma applies this reduction to measures with continuous positive
densities. 

\begin{lemma}[Positive densities on compact convex domains are doubling]\label{thm:bounded}
Let \(X\subset\mathbb R^d\) be bounded and convex, and let \(u\in C(\overline X)\) satisfy \(u>0\) on \(\overline X\). Then the measure \(u(x)\,dx\) is doubling on \(X\).
\end{lemma}

\begin{proof}
Since \(u\) is continuous and strictly positive on the compact set \(\overline X\), there exist constants \(0<m\le M<\infty\) such that \(m\le u\le M\) on \(\overline X\). Let \(\mathcal E=\mathcal E_{E,x}\subset X\) be any ellipsoid. Since
\[
\operatorname{Vol}\!\left(\mathcal E_{\frac12E,x}\right)
=
\left|\det\!\left(\tfrac12 E\right)\right|\operatorname{Vol}(B_1(0))
=
2^{-d}|\det E|\operatorname{Vol}(B_1(0))
=
2^{-d}\operatorname{Vol}(\mathcal E_{E,x}).
\]

Then
\[
\int_{\mathcal E}u(x)\,dx
\le
M\,\operatorname{Vol}(\mathcal E)
=
M\,2^d\,\operatorname{Vol}\!\left(\tfrac12\mathcal E\right)
\le
\frac{M}{m}\,2^d\int_{\frac12\mathcal E}u(x)\,dx.
\]
Thus \(u(x)\,dx\) is doubling for every ellipsoid. Therefore, $u(x)dx$ is doubling on $X$.
\end{proof}

\paragraph{Elliptic case}

We begin with the elliptic Dirichlet problem. The proof splits according to whether the boundary trace vanishes identically or stays strictly positive.

\paragraph{The vanishing boundary case: \(h\equiv 0\) on \(\partial Y\)}
In this regime the solution vanishes at the boundary, so the key step is to compare it with the distance function.

\begin{lemma}[Boundary quotient for the elliptic solution]\label{lemma:Q}
Define \(\widetilde Q\colon \overline Y\to\mathbb R\) by
\[
\widetilde Q(x)=
\begin{cases}
\dfrac{u(x)}{d(x)}, & x\in Y,\\[4pt]
-\partial_n u(x), & x\in \partial Y,
\end{cases}
\]
where \(d(x)=\operatorname{dist}(x,\partial Y)\) and \(n(x)\) denotes the unit outer normal. Then \(\widetilde Q\) is continuous and strictly positive on \(\overline Y\).
\end{lemma}

\begin{proof}
Because \(\partial Y\) is \(C^2\), there exists \(s>0\) such that \(d\in C^2(\{x\in\overline Y:\ d(x)<s\})\). For every \(x\) in this tubular neighborhood, let \(\pi(x)\in\partial Y\) denote the unique nearest boundary point, so that \cite{gilbarg1983elliptic}
\[
x=\pi(x)-d(x)\,n(\pi(x)).
\]
Define
\[
\gamma_x(t):=\pi(x)-t\,d(x)\,n(\pi(x)),\qquad t\in[0,1].
\]
Since \(u=0\) on \(\partial Y\), the fundamental theorem of calculus gives
\[
u(x)
=
u(\gamma_x(1))-u(\gamma_x(0))
=
-d(x)\int_0^1 \nabla u(\gamma_x(t))\cdot n(\pi(x))\,dt.
\]
Hence
\[
\frac{u(x)}{d(x)}
=
-\int_0^1 \nabla u(\gamma_x(t))\cdot n(\pi(x))\,dt.
\]

Let \(x_0\in\partial Y\). As \(x\to x_0\), we have \(d(x)\to 0\) and \(\pi(x)\to x_0\), and therefore \(\gamma_x(t)\to x_0\) uniformly for \(t\in[0,1]\). Since \(\nabla u\) and \(n\) are continuous up to the boundary, the integrand converges uniformly to \(\nabla u(x_0)\cdot n(x_0)\). Consequently,
\[
\lim_{x\to x_0}\frac{u(x)}{d(x)}
=
-\int_0^1 \nabla u(x_0)\cdot n(x_0)\,dt
=
-\partial_n u(x_0).
\]
Thus \(\widetilde Q\) extends continuously to \(\partial Y\).

Positivity in the interior follows from \cref{prop:elliptic_reg_pos}, which gives \(u>0\) in \(Y\) when \(h\equiv 0\). Positivity on the boundary follows from Hopf's lemma \cite{evans2010pde}, which yields $\partial_n u(x)<0 \text{ for every }x\in\partial Y.$ 
Hence \(\widetilde Q\) is continuous and strictly positive on \(\overline Y\).
\end{proof}

\begin{theorem}[Elliptic target measure is doubling]\label{thm:elliptic_doubling}
Assume \cref{ass:elliptic_data}, assume \(Y\) is convex, and suppose either \(h\equiv 0\) on \(\partial Y\) or \(h>0\) on \(\partial Y\). Then the normalized measure \(\tilde u(x)\,dx\) is doubling on \(Y\).
\end{theorem}

\begin{proof}
If \(h>0\) on \(\partial Y\), then \(u\) is continuous and strictly positive on the compact set \(\overline Y\). \cref{thm:bounded} therefore implies that \(u(x)\,dx\) is doubling on \(Y\), and multiplying by the normalizing constant preserves the doubling property.

Assume now that \(h\equiv 0\) on \(\partial Y\). By \cref{lemma:Q}, the function \(\widetilde Q\) is continuous and strictly positive on \(\overline Y\). Hence, by compactness,
\[
0<C_1:=\min_{\overline Y}\widetilde Q
\le
\max_{\overline Y}\widetilde Q=:C_2<\infty.
\]
Recalling that \(u(x)=\widetilde Q(x)\,d(x)\) in \(Y\), we obtain
\[
C_1\,d(x)\le u(x)\le C_2\,d(x)
\qquad\text{for all }x\in Y.
\]

Since \(Y\) is convex, the distance function \(d\) is concave; see, for example, \cite{gilbarg1983elliptic}. Let \(\mathcal E\subset Y\) be any ellipsoid centered at \(x_0\). For every \(x\in \mathcal E\), concavity of \(d\) gives
\[
d\!\left(\frac{x+x_0}{2}\right)\ge \frac12 d(x)+\frac12 d(x_0)\ge \frac12 d(x),
\]
and therefore
\[
d(x)\le 2\,d\!\left(\frac{x+x_0}{2}\right).
\]
Integrating over \(\mathcal E\) and making the change of variables \(y=(x+x_0)/2\), so that \(dx=2^d\,dy\) and \(y\in \frac12\mathcal E\), we obtain
\[
\int_{\mathcal E} d(x)\,dx
\le
2\int_{\mathcal E} d\!\left(\frac{x+x_0}{2}\right)\,dx
=
2^{d+1}\int_{\frac12\mathcal E} d(y)\,dy.
\]
Combining this with the two-sided estimate above yields
\[
\int_{\mathcal E}u(x)\,dx
\le
C_2\int_{\mathcal E}d(x)\,dx
\le
2^{d+1}C_2\int_{\frac12\mathcal E}d(y)\,dy
\le
\frac{2^{d+1}C_2}{C_1}\int_{\frac12\mathcal E}u(y)\,dy.
\]
Thus \(u(x)\,dx\) is doubling on ellipsoids, and \cref{def:doubling} implies that it is doubling on \(Y\). Normalization again preserves the doubling property.
\end{proof}

\paragraph{Parabolic case}

We next consider the parabolic problem. As in the elliptic setting, the argument splits according to the behavior of the terminal profile on the boundary.

\paragraph{The vanishing boundary case:  \(u(\cdot,T)=0\) on \(\partial Y\)}
Write \(u_T(x):=u(x,T)\), and it satisfies 
$
u_T(x)>0\text{ for all }x\in Y.
$

\begin{lemma}[Boundary quotient for the terminal profile]\label{lem:parabolic_Q}
Assume \(u(\cdot,T)=0\) on \(\partial Y\). Define
\[
\widetilde Q_T(x)=
\begin{cases}
\dfrac{u_T(x)}{d(x)}, & x\in Y,\\[6pt]
-\partial_n u(x,T), & x\in \partial Y,
\end{cases}
\]
where \(n(x)\) is the unit outer normal on \(\partial Y\). Then \(\widetilde Q_T\) is continuous and strictly positive on \(\overline Y\).
\end{lemma}

\begin{proof}
Because \(u\in C^{2+\alpha,1+\alpha/2}(\overline{Q_T})\), the terminal profile \(u_T\) belongs to \(C^{2,\alpha}(\overline Y)\). Applying the same normal-ray argument as in \cref{lemma:Q} to the function \(u_T\), we obtain
\[
\frac{u_T(x)}{d(x)}
=
-\int_0^1 \nabla u_T(\gamma_x(t))\cdot n(\pi(x))\,dt
\]
for \(x\) in a tubular neighborhood of \(\partial Y\). Passing to the limit as \(x\to x_0\in\partial Y\) exactly as before yields
\[
\lim_{x\to x_0}\frac{u_T(x)}{d(x)}=-\partial_n u(x_0,T).
\]
Hence \(u_T/d\) extends continuously to \(\partial Y\) with boundary trace \(-\partial_n u(\cdot,T)\).

Since we have \(u_T>0\) in \(Y\), and by assumption \(u_T=0\) on \(\partial Y\). The Hopf boundary lemma in \cite{payne1968maximumprinciples} applied to the terminal profile \(u_T\) therefore gives
\[
\partial_n u(x,T)<0 \qquad \text{for every }x\in\partial Y.
\]
It follows that \(\widetilde Q_T\) is continuous and strictly positive on \(\overline Y\).
\end{proof}

\begin{theorem}[Parabolic target measure is doubling]\label{thm:parabolic_doubling}
Assume \cref{ass:parabolic_data}, assume \(Y\) is convex, and suppose either \(u(\cdot,T)=0\) on \(\partial Y\) or \(u(\cdot,T)>0\) on \(\partial Y\). Then the terminal-time target measure \(\tilde u(\cdot,T)\,dx\) is doubling on \(Y\).
\end{theorem}

\begin{proof}
If \(u(\cdot,T)>0\) on \(\partial Y\), then continuity on the compact set \(\overline Y\) yields
\[
0<m_T:=\min_{\overline Y}u(\cdot,T)\le \max_{\overline Y}u(\cdot,T)=:M_T<\infty.
\]
Therefore \(u(\cdot,T)\,dx\) is doubling by \cref{thm:bounded}, and normalization preserves doubling.

Assume now that \(u(\cdot,T)=0\) on \(\partial Y\). By \cref{lem:parabolic_Q}, there exist constants \(0<C_1\le C_2<\infty\) such that
\[
C_1\,d(x)\le u_T(x)\le C_2\,d(x)
\qquad\text{for all }x\in Y.
\]
Since \(Y\) is convex, the same concavity argument used in the elliptic case yields
\[
\int_{\mathcal E}u_T(x)\,dx
\le
\frac{2^{d+1}C_2}{C_1}\int_{\frac12\mathcal E}u_T(y)\,dy
\]
for every ellipsoid \(\mathcal E\subset Y\). Thus \(u_T(x)\,dx\) is doubling on ellipsoids, hence doubling on \(Y\), and the normalized measure \(\tilde u(\cdot,T)\,dx\) is doubling as well.
\end{proof}

\paragraph{Torus case}

We finally consider the periodic Fokker--Planck equation \eqref{eq:FP} and
the stationary invariant-density problem
\eqref{eq:stationary-invariant-density}. Since there is no boundary, the
argument is simpler: for the time-dependent solution we first prove strict
positivity for every positive time, while the stationary invariant density is
positive by the elliptic strong maximum principle, and then both cases follow
from compactness of a fundamental domain.

\begin{theorem}[Strict positivity for positive times]\label{thm:strictpos}
Let \(u\) be the solution of \eqref{eq:FP}. Then, for every \(t>0\) and every
\(x\in\mathbb T^d\), one has \(u(t,x)>0\). Moreover, if \(m\) solves the
stationary invariant-density problem \eqref{eq:stationary-invariant-density},
then \(m(x)>0\) for every \(x\in\mathbb T^d\).
\end{theorem}

\begin{proof}
Assume, for contradiction, that there exist \(t_*>0\) and
\(x_*\in\mathbb T^d\) such that \(u(t_*,x_*)=0\). Fix \(0<\tau<t_*\). Since
\(u\) solves \(\partial_t u-\mathcal L^*u=0\) on
\((\tau,t_*+\varepsilon)\times\mathbb T^d\) for sufficiently small
\(\varepsilon>0\), the point \((t_*,x_*)\) is an interior point of this
parabolic cylinder. Moreover, \(u\ge 0\) everywhere. The strong maximum
principle therefore implies that \(u(t,x)\equiv 0\) on
\([\tau,t_*]\times\mathbb T^d\). Since \(\tau\in(0,t_*)\) was arbitrary, we
obtain \(u(t_*,x)\equiv 0\) for all \(x\in\mathbb T^d\). Hence $\int_{\mathbb T^d}u(t_*,x)\,dx=0,$
which contradicts the mass conservation statement in \cref{prop:fp_smooth}.
Therefore no such point \((t_*,x_*)\) can exist, and \(u(t,x)>0\) for every
\(t>0\) and every \(x\in\mathbb T^d\).

It remains to prove the assertion for \(m\). Since \(m\) solves
\(0=\Delta(\kappa m)-\nabla\cdot(\bm b m)\), satisfies \(m\ge 0\), and is
normalized by \(\int_{\mathbb T^d}m(x)\,dx=1\), it is not identically zero. If
there were a point \(x_*\in\mathbb T^d\) such that \(m(x_*)=0\), then the
elliptic strong maximum principle applied on the compact manifold
\(\mathbb T^d\) would imply \(m\equiv 0\), contradicting the normalization.
Thus \(m(x)>0\) for every \(x\in\mathbb T^d\).
\end{proof}

Let \(D:=[0,L]^d\), and let \(\pi:D\to\mathbb T^d\) denote the canonical
quotient map. We continue to write \(dx\) for the normalized Lebesgue measure
on \(D\), so that \(\pi_\sharp(dx)=dx\) on \(\mathbb T^d\). For \(t>0\),
define the periodic representative of \(u\) on \(D\) by
\begin{equation}\label{eq:pullback}
\hat u(t,x):=u(t,\pi(x)),
\qquad (t,x)\in(0,\infty)\times D.
\end{equation}
In the invariant-density case, define similarly
\(\hat m(x):=m(\pi(x))\).

\begin{lemma}[Periodic representative on the cube]\label{prop:positive_on_D}
For every \(t>0\), the function \(\hat u(t,\cdot)\) is continuous and strictly
positive on \(D\), and \(\int_D \hat u(t,x)\,dx=1\). The same conclusion holds
with \(\hat u(t,\cdot)\) replaced by \(\hat m\) in the invariant-density case.
\end{lemma}

\begin{proof}
Continuity and positivity follow immediately from the definition
\eqref{eq:pullback} and \cref{thm:strictpos}. Since \(dx\) denotes the
normalized Lebesgue measure on \(D\) and \(\pi_\sharp(dx)=dx\) on
\(\mathbb T^d\), we have
\[
\int_D \hat u(t,x)\,dx
=
\int_{\mathbb T^d}u(t,y)\,dy
=
1
\]
by \cref{prop:fp_smooth}. The proof for \(\hat m\) is identical, using the
smoothness, positivity, and normalization of \(m\).
\end{proof}
Then we can use this lemma to prove that the torus target measure satisfies the doubling condition.
\begin{theorem}[Torus target measure is doubling]\label{thm:torus_doubling}
For every fixed \(t>0\), the measure \(\hat u(t,\cdot)\,dx\) is doubling on
\(D\). The same conclusion holds for \(\hat m\,dx\) in the invariant-density
case.
\end{theorem}

\begin{proof}
By \cref{prop:positive_on_D}, the function \(\hat u(t,\cdot)\), or \(\hat m\)
in the invariant-density case, is continuous and strictly positive on the
compact convex cube \(D\). \cref{thm:bounded} therefore implies that the
corresponding measure is doubling on \(D\).
\end{proof}

We now combine the doubling results above with the regularity theory for
optimal transport.

\begin{proposition}[Regularity of optimal transport under doubling measures, {\normalfont\cite{jhaveri2022regularity}}]
\label{lemma:caffarelli}
Let \(X\) and \(Y\) be open, bounded convex sets in \(\mathbb{R}^d\), and suppose that \(f\,dx\) and \(g\,dx\) are doubling measures concentrated on \(X\) and \(Y\), respectively. Let \(\mathcal{T}\) be the optimal transport map sending \(f\,dx\) to \(g\,dx\). Then
\[
\mathcal{T} \in C^\sigma(\overline{X})
\]
for some \(\sigma\in(0,1)\), where \(\sigma\) depends only on \(d\), the doubling constants of \(f\,dx\) and \(g\,dx\), and the inner and outer diameters of \(X\) and \(Y\).
\end{proposition}

\begin{proof}[Proof of \cref{thm:main1}]
The source measure \(\mu\) is the uniform probability measure on the bounded convex set \(X\), so it is doubling by \cref{thm:bounded}. For the target measure \(\nu\), the doubling property is provided by \cref{thm:elliptic_doubling} in the elliptic case, \cref{thm:parabolic_doubling} in the parabolic case, and \cref{thm:torus_doubling} in the torus case after identifying \(\mathbb T^d\) with the cube \(D\).

Thus both the source and target measures are doubling measures supported on bounded convex domains. Applying \cref{lemma:caffarelli}, we conclude that the optimal transport map \(\mathcal{T}\) belongs to \(C^\alpha(\overline X)\) for some \(\alpha\in(0,1]\). This is exactly the claimed regularity.
\end{proof}

\subsection{Proof of \cref{thm:main-excess-risk}}\label{sec:excess-risk}
We first decompose the excess risk into four terms, and then we estimate each term separately. The approximation term is controlled by the approximation theory of ReLU networks, while the generalization and discretization terms are bounded using empirical Wasserstein bounds. The optimization term is left as an abstract quantity that depends on the training procedure. 
Let \(\mathcal H\) be the hypothesis class induced by the neural parameterization. On the empirical measures, define the best-in-class empirical optimizer
\[
\mathcal{T}_N^\ast \in \arg\min_{S\in\mathcal H} R(S;\hat\mu_N,\hat\nu_N),
\]
and let \(\mathcal{T}_{\theta,N}\in\mathcal H\) denote the output of the practical training procedure. Then we introduce the following decomposition with four terms:

\begin{proposition}[Decomposition]
\label{prop:excess-risk-decomposition}
For every $N$,
\[
R(\mathcal{T}_{\theta,N};\mu,\nu)-R(\mathcal{T};\mu,\nu)
=
\varepsilon_{\mathrm{gen},N}
+\varepsilon_{\mathrm{opt},N}
+\varepsilon_{\mathrm{app},N}
+\varepsilon_{\mathrm{disc},N}.
\]
where \begin{align*}
\varepsilon_{\mathrm{gen},N} &\coloneqq R(\mathcal{T}_{\theta,N};\mu,\nu)-R(\mathcal{T}_{\theta,N};\hat\mu_N,\hat\nu_N),\\
\varepsilon_{\mathrm{opt},N} &\coloneqq R(\mathcal{T}_{\theta,N};\hat\mu_N,\hat\nu_N)-R(\mathcal{T}_N^\ast;\hat\mu_N,\hat\nu_N),\\
\varepsilon_{\mathrm{app},N} &\coloneqq R(\mathcal{T}_N^\ast;\hat\mu_N,\hat\nu_N)-R(\mathcal{T};\hat\mu_N,\hat\nu_N),\\
\varepsilon_{\mathrm{disc},N} &\coloneqq R(\mathcal{T};\hat\mu_N,\hat\nu_N)-R(\mathcal{T};\mu,\nu).
\end{align*}

\end{proposition}

\begin{proof}
This is the telescoping identity obtained by adding and subtracting $R(\mathcal{T}_{\theta,N};\hat\mu_N,\hat\nu_N)$, $R(\mathcal{T}_N^\ast;\hat\mu_N,\hat\nu_N)$, and $R(\mathcal{T};\hat\mu_N,\hat\nu_N)$.
\end{proof}

The four terms have the following roles:
$\varepsilon_{\mathrm{gen},N}$ measures the generalization from empirical to
population risk for the neural network $\mathcal{T}_\theta$,
$\varepsilon_{\mathrm{opt},N}$ records the optimization error of the training
procedure, $\varepsilon_{\mathrm{app},N}$ is the approximation error of the
hypothesis class, and $\varepsilon_{\mathrm{disc},N}$, with ``disc'' short for
discretization, captures the discretization error caused by replacing the
population measures $\mu,\nu$ with their empirical counterparts
$\hat\mu_N,\hat\nu_N$ when the map is fixed at the population map
$\mathcal{T}$. We begin with the approximation term.

\begin{lemma}[Approximation term]
\label{prop:approximation-error}
For every $S\in\mathcal H$,
\[
\varepsilon_{\mathrm{app},N}\le \Bigl(\int_X \|\mathcal{T}(x)-S(x)\|^2\,d\hat\mu_N(x)\Bigr)^{1/2}.
\]
In particular, if $\|\mathcal{T}-S\|_{L^\infty(X)}\le \delta_N$, then $\varepsilon_{\mathrm{app},N}\le \delta_N$.
\end{lemma}

\begin{proof}
By the minimizing property of $\mathcal{T}_N^\ast$, for every $S\in\mathcal H$,
\[
\varepsilon_{\mathrm{app},N}
=
R(\mathcal{T}_N^\ast;\hat\mu_N,\hat\nu_N)-R(\mathcal{T};\hat\mu_N,\hat\nu_N)
\le
R(S;\hat\mu_N,\hat\nu_N)-R(\mathcal{T};\hat\mu_N,\hat\nu_N).
\]
Using the reverse triangle inequality for $W_2$, we obtain
\[
R(S;\hat\mu_N,\hat\nu_N)-R(\mathcal{T};\hat\mu_N,\hat\nu_N)
=
W_2(S_\sharp\hat\mu_N,\hat\nu_N)-W_2(\mathcal{T}_\sharp\hat\mu_N,\hat\nu_N)
\le
W_2(S_\sharp\hat\mu_N,\mathcal{T}_\sharp\hat\mu_N).
\]
Finally, the coupling $(S,\mathcal{T})_\sharp\hat\mu_N$ gives
\[
W_2(S_\sharp\hat\mu_N,\mathcal{T}_\sharp\hat\mu_N)
\le
\Bigl(\int_X \|\mathcal{T}(x)-S(x)\|^2\,d\hat\mu_N(x)\Bigr)^{1/2}.
\]
This proves the stated estimate, and the $L^\infty$ consequence is immediate.
\end{proof}

We next invoke a quantitative approximation theorem for H\"older functions.

\begin{proposition}[Approximation of H\"older functions by ReLU networks {\normalfont\cite{shen2020neurons}}]
\label{lem:holder-approximation}
Let $f$ be an $\alpha$-H\"older continuous function with a H\"older constant $\lambda$ on a set $E\subseteq[-B,B]^d$. For arbitrary $L,W\in\mathbb N^+$, there exists a ReLU feedforward neural network $\phi$ with width $3^{d+3}\max\{d\lfloor W^{1/d}\rfloor,\,W+1\}$ and depth $12L+14+2d$ such that
\[
\|f-\phi\|_{L^\infty(E)}\le 19\sqrt d\,\lambda(2B)^\alpha W^{-2\alpha/d}L^{-2\alpha/d}.
\]
\end{proposition}
The preceding proposition gives a general ReLU approximation result for
H\"older functions. Applying it to our approximation term yields the following approximation rate.
\begin{lemma}[Approximation rate]
\label{prop:approximation-rate}
Assume that $\mathcal{T}\in C^\alpha(\overline X)$. Let $L,W\in\mathbb N^+$, and suppose that $\mathcal H$ contains the architectures from \cref{lem:holder-approximation}. Then there exists $S_{W,L}\in\mathcal H$ such that
\[
\|\mathcal{T}-S_{W,L}\|_{L^\infty(\operatorname{supp}\hat\mu_N)}
\le
C_{\alpha,d,B,\mathcal{T}}\,W^{-2\alpha/d}L^{-2\alpha/d}.
\]
Consequently,
\[
\varepsilon_{\mathrm{app},N}\le C_{\alpha,d,B,\mathcal{T}}\,W^{-2\alpha/d}L^{-2\alpha/d}.
\]
\end{lemma}

\begin{proof}
Apply \cref{lem:holder-approximation} to $\mathcal{T}$ and then use \cref{prop:approximation-error}.
\end{proof}

We next control the discretization term.

\begin{lemma}[Discretization term]
\label{prop:discretization-error}
Assume that $\mathcal{T}_\sharp\mu=\nu$. Then
\[
\varepsilon_{\mathrm{disc},N}=W_2(\mathcal{T}_\sharp\hat\mu_N,\hat\nu_N)
\le
W_2(\mathcal{T}_\sharp\hat\mu_N,\nu)+W_2(\nu,\hat\nu_N).
\]
Moreover, $\mathcal{T}_\sharp\hat\mu_N$ has the same law as an empirical measure of $\nu$, and therefore
\[
\mathbb E[\varepsilon_{\mathrm{disc},N}]\le 2\,\mathbb E[W_2(\nu,\hat\nu_N)].
\]
\end{lemma}

\begin{proof}
Since $\mathcal{T}_\sharp\mu=\nu$, we have
\[
R(\mathcal{T};\mu,\nu)=W_2(\mathcal{T}_\sharp\mu,\nu)=0.
\]
Therefore, by the triangle inequality,
\[
\varepsilon_{\mathrm{disc},N}
=
R(\mathcal{T};\hat\mu_N,\hat\nu_N)
=
W_2(\mathcal{T}_\sharp\hat\mu_N,\hat\nu_N)
\le
W_2(\mathcal{T}_\sharp\hat\mu_N,\nu)+W_2(\nu,\hat\nu_N).
\]
For the expectation estimate, note that $\mathcal{T}_\sharp\hat\mu_N=\frac1N\sum_{i=1}^N\delta_{\mathcal{T}(X_i)},$ and the random variables $\mathcal{T}(X_i)$ are i.i.d. with law $\nu$. Let $\hat\nu_N'$ be an empirical measure built from an i.i.d. sample of $\nu$. Then $\mathcal{T}_\sharp\hat\mu_N$ has the same law as $\hat\nu_N'$, and hence
\[
\mathbb E[W_2(\mathcal{T}_\sharp\hat\mu_N,\nu)]
=
\mathbb E[W_2(\hat\nu_N',\nu)]
=
\mathbb E[W_2(\hat\nu_N,\nu)].
\]
Therefore,
\[
\mathbb E[\varepsilon_{\mathrm{disc},N}]
\le
\mathbb E[W_2(\mathcal{T}_\sharp\hat\mu_N,\nu)]
+
\mathbb E[W_2(\nu,\hat\nu_N)]
=
2\,\mathbb E[W_2(\nu,\hat\nu_N)].
\]
\end{proof}

We then turn to the generalization term. The key input is a uniform
Lipschitz bound for the hypothesis class. We first record a standard
estimate showing that bounded network parameters, for a fixed
architecture, imply such a uniform Lipschitz bound. This type of
bounded-parameter or uniform
Lipschitz bound hypothesis class is standard in
generalization analyses of one-step generative models, for instance,
\cite{chen2022distributiongan,chakraborty2025generative}.
\begin{lemma}[Lipschitz bound for ReLU networks]
\label{lem:relu-lipschitz}
Let $S:\mathbb R^{d_0}\to\mathbb R^{d_L}$ be a fixed-architecture
feedforward ReLU network. Suppose that for $\ell=1,2\dots L$, the network's parameters satisfy $\|A_\ell\|,\ \|b_\ell\|\le \kappa,$
for some $\kappa>0$. Then $S$ is Lipschitz. That is
\[
\operatorname{Lip}(S)\le L_{\mathcal H}.
\]
\end{lemma}

\begin{proof}
Each affine map is Lipschitz, and the ReLU activation is
$1$-Lipschitz. Since the architecture is fixed and the weights are
uniformly bounded, the product of the layerwise Lipschitz constants is
uniformly bounded. Hence $S$ is Lipschitz.
\end{proof}

Although $L_{\mathcal H}$ may be large in practice, making the estimate
loose at the level of constants, this does not affect the convergence rate in
$N$ since the hypothesis class is fixed and $L_{\mathcal H}$ is
independent of $N$. Moreover, $L_{\mathcal H}$ can often be controlled
by weight clipping \cite{arjovsky2017wgan}, spectral normalization \cite{miyato2018spectralnormalizationgenerativeadversarial} or
near-identity parameterizations \cite{shen2026twostepdiffusion}.
\begin{lemma}[Generalization term]
\label{prop:generalization-error}
Assume that $\operatorname{Lip}(S)\le L_{\mathcal H}$ for all $S\in\mathcal H$. Then
\[
\varepsilon_{\mathrm{gen},N}\le L_{\mathcal H}W_2(\mu,\hat\mu_N)+W_2(\nu,\hat\nu_N).
\]
\end{lemma}

\begin{proof}
Using the elementary estimate
\[
|W_2((\mathcal{T}_{\theta,N})_\sharp\mu,\nu)-W_2((\mathcal{T}_{\theta,N})_\sharp\hat\mu_N,\hat\nu_N)|
\le
W_2((\mathcal{T}_{\theta,N})_\sharp\mu,(\mathcal{T}_{\theta,N})_\sharp\hat\mu_N)+W_2(\nu,\hat\nu_N),
\]
we obtain
\[
\varepsilon_{\mathrm{gen},N}
\le
W_2((\mathcal{T}_{\theta,N})_\sharp\mu,(\mathcal{T}_{\theta,N})_\sharp\hat\mu_N)
+
W_2(\nu,\hat\nu_N).
\]
Since pushforward by an $L_{\mathcal H}$-Lipschitz map is $L_{\mathcal H}$-Lipschitz in $W_2$,
\[
W_2((\mathcal{T}_{\theta,N})_\sharp\mu,(\mathcal{T}_{\theta,N})_\sharp\hat\mu_N)
\le
\operatorname{Lip}(\mathcal{T}_{\theta,N})\,W_2(\mu,\hat\mu_N)
\le
L_{\mathcal H}W_2(\mu,\hat\mu_N).
\]
This proves the claim.
\end{proof}

We are now ready to prove \cref{thm:main-excess-risk}.

\begin{proof}[Proof of \cref{thm:main-excess-risk}]
By \cref{prop:excess-risk-decomposition},
\[
R(\mathcal{T}_{\theta,N};\mu,\nu)-R(\mathcal{T};\mu,\nu)
=
\varepsilon_{\mathrm{gen},N}
+\varepsilon_{\mathrm{opt},N}
+\varepsilon_{\mathrm{app},N}
+\varepsilon_{\mathrm{disc},N}.
\]
Taking expectations and using \cref{prop:generalization-error,prop:approximation-rate,prop:discretization-error}, we obtain
\begin{align*}
\mathbb E\!\left[R(\mathcal{T}_{\theta,N};\mu,\nu)-R(\mathcal{T};\mu,\nu)\right]
\le {}&
L_{\mathcal H}\,\mathbb E[W_2(\mu,\hat\mu_N)]
+
\mathbb E[W_2(\nu,\hat\nu_N)] \\
&+
\mathbb E[\varepsilon_{\mathrm{opt},N}]
+
C_{\alpha,d,B,\mathcal{T}}\,W^{-2\alpha/d}L^{-2\alpha/d}
+
2\,\mathbb E[W_2(\nu,\hat\nu_N)].
\end{align*}
Collecting the two terms involving $W_2(\nu,\hat\nu_N)$ yields the claimed bound.
\end{proof}
Then we can deduce the sample-size dependence for fixed architectures.
\begin{proof}[Proof of \cref{cor:fixed-arch-rate}]
Choose $W_\eta$ and $L_\eta$ so large that $C_{\alpha,d,B,\mathcal{T}}W_\eta^{-2\alpha/d}L_\eta^{-2\alpha/d}\le\eta$, and let $\mathcal H_\eta$ be the corresponding hypothesis class. This is possible because $W^{-2\alpha/d}L^{-2\alpha/d}\to0$ as $W,L\to\infty$. Applying \cref{thm:main-excess-risk} with the fixed class $\mathcal H_\eta$ and $\mathbb E[\varepsilon_{\mathrm{opt},N}]\le\eta$, we reduce the problem to bounding $\mathrm{Stat}_N(\mathcal H_\eta)$.

Since $\mu,\nu$ is supported in the bounded set $X\subset[-B,B]^d$, it belongs to $\mathcal P_q(\mathbb R^d)$ for every $q>0$. Therefore the empirical Wasserstein bounds in \cite{fournier2015wasserstein} with $p=2$ imply (The following bound also holds for $\nu$)
\[
\mathbb E[W_2^2(\mu,\hat\mu_N)]
\le
C_{\eta}
\begin{cases}
N^{-1/2}, & d<4,\\
N^{-1/2}\log(1+N), & d=4,\\
N^{-2/d}, & d>4.
\end{cases}
\]
By Jensen's inequality $\mathbb E[W_2(\mu,\hat\mu_N)] \le \bigl(\mathbb E[W_2^2(\mu,\hat\mu_N)]\bigr)^{1/2}$, 
and substituting these bounds into $\mathrm{Stat}_N(\mathcal H_\eta)$ gives the stated estimate after absorbing constants into $C_{\eta}$.
\end{proof}

\subsection{Out-of-distribution regime}\label{sec:ood}

We close the section with a simple robustness estimate under target shift. The source measure $\mu$ is kept fixed, while the target measure $\nu$ is replaced by $\nu_{1}$. The resulting out-of-distribution (OOD) excess risk is $R(\mathcal{T}_{\theta,N};\mu,\nu_{1})-R(\mathcal{T};\mu,\nu)$.

\begin{lemma}[OOD bound under target shift]
\label{prop:ood-bound}
Assume that the source distribution $\mu$ remains unchanged and the target distribution shifts from $\nu$ to $\nu_{1}$. Then
\[
R(\mathcal{T}_{\theta,N};\mu,\nu_{1})-R(\mathcal{T};\mu,\nu)
\le
\bigl(R(\mathcal{T}_{\theta,N};\mu,\nu)-R(\mathcal{T};\mu,\nu)\bigr)+W_2(\nu,\nu_{1}).
\]
Consequently, under the assumptions of \cref{thm:main-excess-risk},
\[
\mathbb E\!\left[R(\mathcal{T}_{\theta,N};\mu,\nu_{1})-R(\mathcal{T};\mu,\nu)\right]
\le
\mathrm{Stat}_N(\mathcal H)+\mathrm{App}(\mathcal H;\mathcal{T})+\mathbb E[\varepsilon_{\mathrm{opt},N}]+W_2(\nu,\nu_{1}).
\]
\end{lemma}

\begin{proof}
Decompose
\[
R(\mathcal{T}_{\theta,N};\mu,\nu_{1})-R(\mathcal{T};\mu,\nu)
=
\bigl(R(\mathcal{T}_{\theta,N};\mu,\nu_{1})-R(\mathcal{T}_{\theta,N};\mu,\nu)\bigr)
+
\bigl(R(\mathcal{T}_{\theta,N};\mu,\nu)-R(\mathcal{T};\mu,\nu)\bigr).
\]
For the first term, the reverse triangle inequality for $W_2$ gives
\[
W_2((\mathcal{T}_{\theta,N})_\sharp\mu,\nu_{1})
-
W_2((\mathcal{T}_{\theta,N})_\sharp\mu,\nu)
\le
W_2(\nu,\nu_{1}),
\]
that is,
\[
R(\mathcal{T}_{\theta,N};\mu,\nu_{1})-R(\mathcal{T}_{\theta,N};\mu,\nu)
\le
W_2(\nu,\nu_{1}).
\]
Combining the two estimates and invoking \cref{thm:main-excess-risk} proves the result.
\end{proof}

\section{Numerical Results}\label{sec:numerical}

The purpose of this section is to provide simple numerical consistency checks for the preceding excess-risk analysis. We first present explicit one and two dimensional PDE-induced target measures for which the optimal transport maps can be computed in closed form, and then report numerical experiments verifying the predicted sample-size dependence.

\subsection{One- and two- dimensional problems}

\begin{example}[One-dimensional model problem]\label{ex:1d-model}
Let \(\Omega=[0,1]\) and consider
\begin{equation}\label{eq:1d-pde-positive}
u''(x)=0 \quad \text{in }(0,1),
\qquad
u(0)=\tfrac12,\quad u(1)=\tfrac32.
\end{equation}
\end{example}

Let \(\mu\) denote the uniform probability measure on \([0,1]\), and set \(\nu=u(x)\,dx\). Then \(u(x)=x+\tfrac12\), so \(\nu\) is a probability measure on \([0,1]\). 
\begin{corollary}
For the problem in \cref{ex:1d-model}, the quadratic-cost optimal transport map \(\mathcal{T}:[0,1]\to[0,1]\) from \(\mu\) to \(\nu\) is $\mathcal{T}(x)=\frac{-1+\sqrt{1+8x}}{2}$
and \(\mathcal{T}\in C^{0,1}([0,1])\). 
Let \(\hat\mu_N\) and \(\hat\nu_N\) are the corresponding empirical measures based on \(N\) i.i.d. samples, then
$\mathbb E\!\left[W_2(\mu,\hat\mu_N)\right]\le \frac{1}{\sqrt{3(N+1)}}$, $\mathbb E\!\left[W_2(\nu,\hat\nu_N)\right]\le \sqrt{\frac{20-9\ln 3}{32(N+1)}}.$ 
Consequently, for every \(\eta>0\), one may choose a hypothesis class \(\mathcal H_\eta\) with width and depth independent of \(N\) so that \(\mathrm{App}(\mathcal H_\eta;\mathcal{T})\le \eta\); if, in addition, \(\mathbb E[\varepsilon_{\mathrm{opt},N}]\le \eta\), then there exists \(C_\eta>0\) such that
\[
\mathbb E\!\left[R(\mathcal{T}_{\theta,N};\mu,\nu)-R(\mathcal{T};\mu,\nu)\right]
\le C_\eta N^{-1/2}+2\eta.
\]
\end{corollary}

\begin{proof}
The differential equation implies that \(u\) is affine, and the boundary conditions give \(u(x)=x+\tfrac12\). In particular, \(u>0\) on \([0,1]\) and \(\int_0^1 u(x)\,dx=1\), so \(\nu\) is a probability measure. Since this is a special case of the elliptic setting of \cref{thm:main1} with \(g\equiv0\) and strictly positive boundary data, the associated optimal transport map is H\"older continuous. In one dimension the map can be computed explicitly: the distribution functions are \(F_\mu(x)=x\) and \(F_\nu(y)=\tfrac12(y^2+y)\), hence \(\mathcal{T}=F_\nu^{-1}\circ F_\mu\), which gives the stated formula. Moreover, \(\mathcal{T}'(x)=2(1+8x)^{-1/2}\le 2\), so \(\mathcal{T}\in C^{0,1}([0,1])\).

For the empirical Wasserstein terms we use the sharper bound in \cite{bobkov2019empirical}, which states that \(\mathbb E[W_2^2(\lambda,\hat\lambda_N)]\le 2J_2(\lambda)/(N+1)\), where \(J_2(\lambda)=\int \frac{F_\lambda(x)(1-F_\lambda(x))}{f_\lambda(x)}\,dx\) whenever \(\lambda\) has density \(f_\lambda\) and distribution function \(F_\lambda\). Here \(J_2(\mu)=\int_0^1 x(1-x)\,dx=\tfrac16\), while for \(\nu\), since \(f_\nu(x)=x+\tfrac12\) and \(F_\nu(x)=\tfrac12(x^2+x)\), one computes $
J_2(\nu)=\int_0^1 \frac{F_\nu(x)(1-F_\nu(x))}{f_\nu(x)}\,dx=\frac{20-9\ln 3}{64}.$
Jensen's inequality then yields the stated bounds. The final estimate follows from \cref{cor:fixed-arch-rate}.
\end{proof}

\begin{example}[Two-dimensional model problem]\label{ex:2d-model}
Let \(\Omega=B_1(0)=\{x\in\mathbb R^2:\ |x|\le 1\}\) and consider
\begin{equation}\label{eq:2d-pde-positive}
-\Delta u = 1 \quad \text{in } B_1(0),
\qquad
u=0 \quad \text{on } \partial B_1(0).
\end{equation}
Let \(\mu\) denote the uniform probability measure on \(B_1(0)\),$\mu=\frac1\pi \mathbf{1}_{B_1(0)}(x)\,dx,$ and let \(\nu\) be the probability measure obtained by normalizing \(u\). Then \(u(x)=\frac{1-|x|^2}{4}\), and therefore $\nu=\frac{2}{\pi}(1-|x|^2)\mathbf{1}_{B_1(0)}(x)\,dx.$ 
\end{example}
\begin{corollary}
For the problem in \cref{ex:1d-model}, the quadratic-cost optimal transport map \(\mathcal{T}:B_1(0)\to B_1(0)\) from \(\mu\) to \(\nu\) is H\"older continuous since this is a special case of \cref{thm:main1}. 
Consequently, for every \(\eta>0\), one may choose a hypothesis class \(\mathcal H_\eta\) with width and depth independent of \(N\) so that \(\mathrm{App}(\mathcal H_\eta;\mathcal{T})\le \eta\); if, in addition, \(\mathbb E[\varepsilon_{\mathrm{opt},N}]\le \eta\), then by \cref{cor:fixed-arch-rate}, there exists \(C_\eta>0\) such that
\[
\mathbb E\!\left[R(\mathcal{T}_{\theta,N};\mu,\nu)-R(\mathcal{T};\mu,\nu)\right]
\le C_\eta N^{-1/4}+2\eta.
\]
\end{corollary}

\subsection{Numerical verification of the predicted rate}

We now give numerical consistency checks for \cref{ex:1d-model,ex:2d-model}, together with \cref{cor:fixed-arch-rate}. Throughout all experiments the architecture is kept fixed, in accordance with the theoretical discussion: the learned map \(\mathcal{T}_{\theta,N}\) is parameterized by a multilayer perceptron whose hypothesis class does not vary with the sample size \(N\). Thus only the empirical measures \(\hat\mu_N\) and \(\hat\nu_N\) change as \(N\) increases.

\begin{figure}[t]
\centering
\begin{minipage}[t]{0.48\textwidth}
\centering
\includegraphics[width=\linewidth]{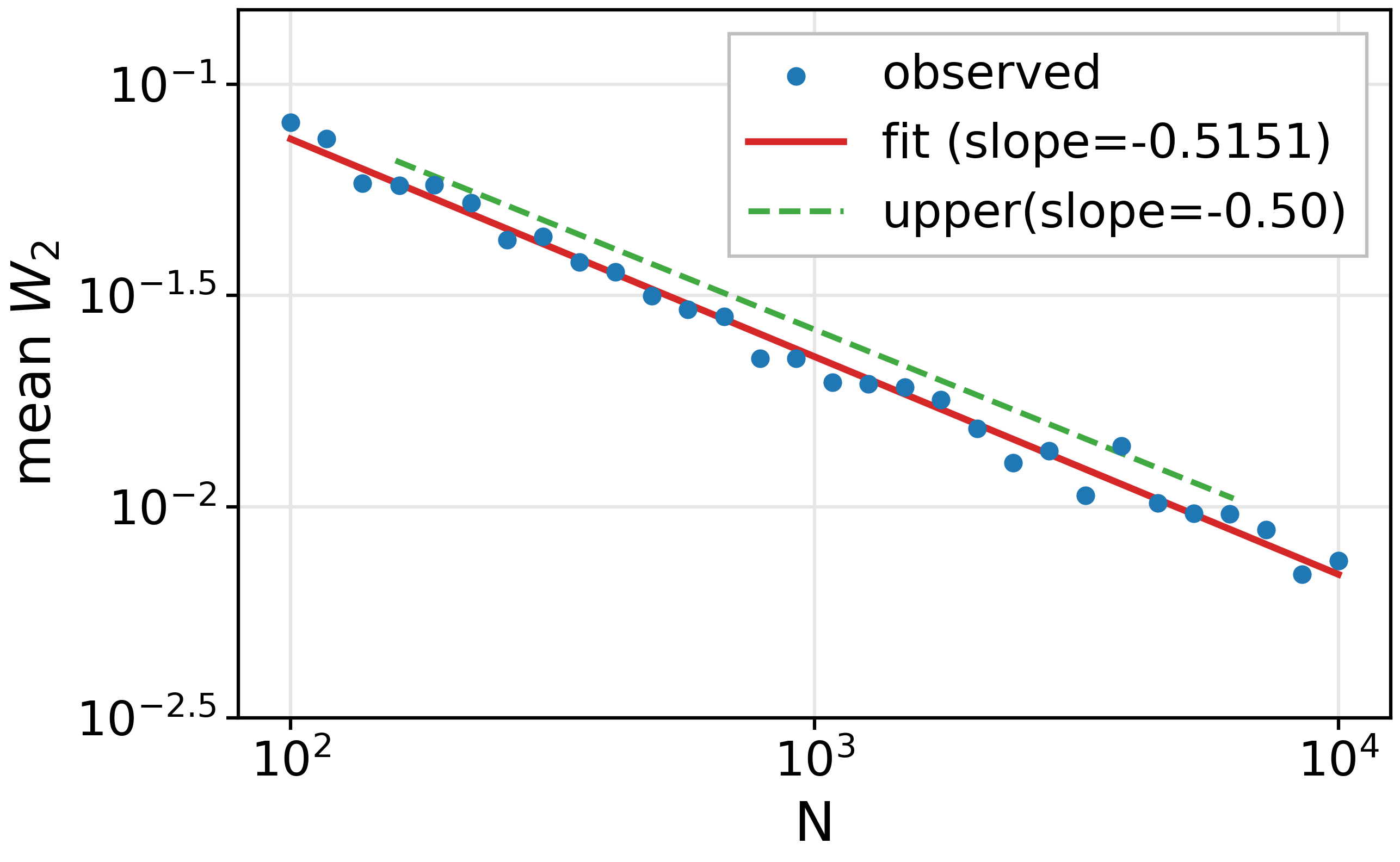}

{\small (a) One-dimensional experiment.}
\end{minipage}\hfill
\begin{minipage}[t]{0.48\textwidth}
\centering
\includegraphics[width=\linewidth]{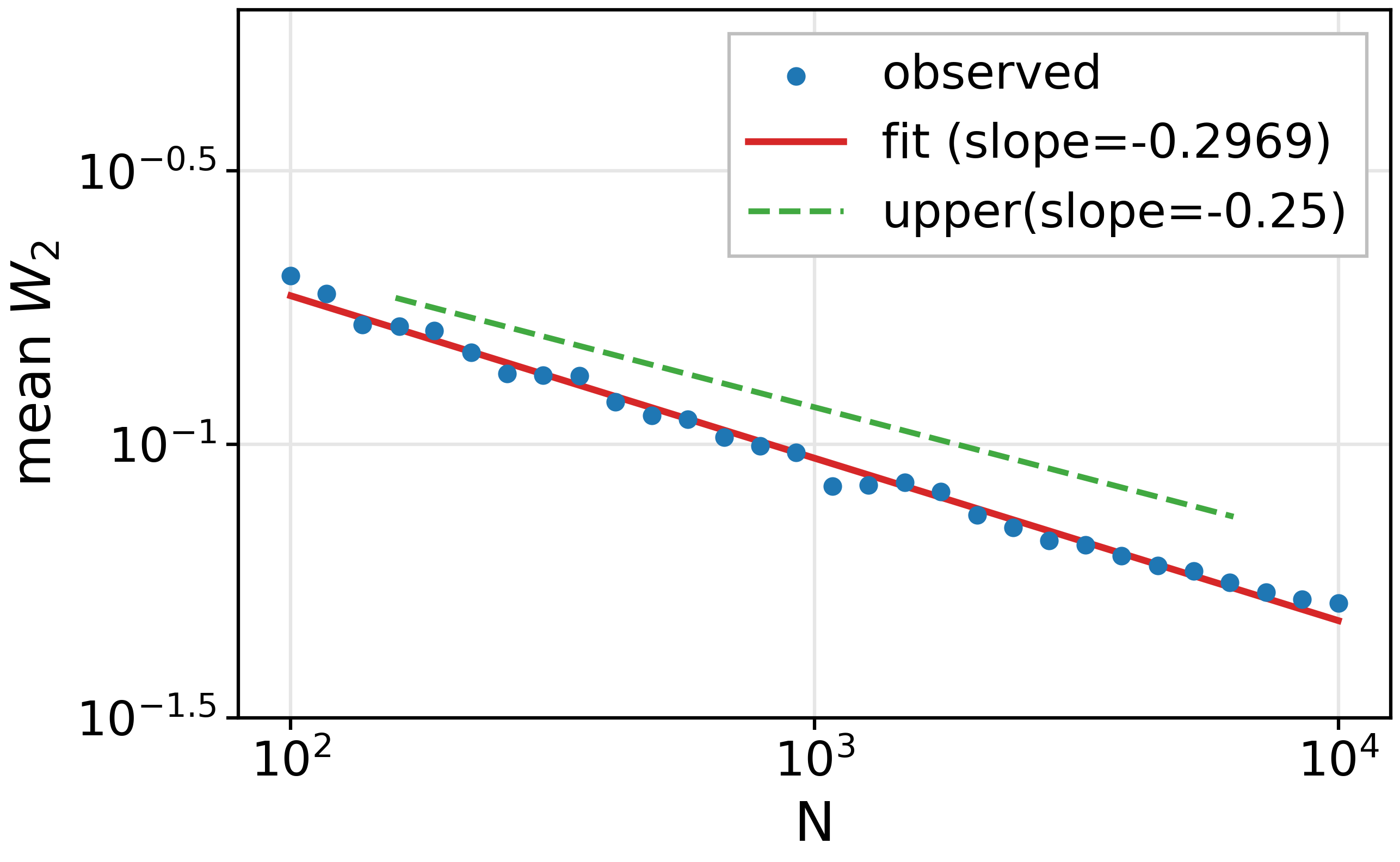}

{\small (b) Two-dimensional experiment.}
\end{minipage}
\caption{Validation Wasserstein--2 error versus sample size for the one- and two-dimensional model problems. In each panel, the blue curve shows the observed mean validation error and the red line is the least-squares fit in log-log coordinates.}
\label{fig:numerical-rates}
\end{figure}

\paragraph{One-dimensional experiment}
We use the source and target measures from \cref{ex:1d-model} and the exact optimal transport map given there. The network architecture is \([1,256,256,1]\). We consider \(30\) logarithmically spaced sample sizes \(N\in[10^2,10^4]\), and for each \(N\) we train \(30\) independent models. Optimization is performed with Adam, initial learning rate \(10^{-2}\), and a StepLR scheduler with step size \(500\) and decay factor \(0.9\). Each run is trained for at most \(10^5\) iterations with mini-batch size \(\lfloor N/2\rfloor\), and early stopping with patience \(5000\) is employed. Source samples are drawn from \(\mu\), while target samples are generated from \(\nu\) by inverse-transform sampling.

For evaluation, we use an independent validation set of size \(10^6\). For each run we retain the iterate with smallest training loss and compute the one-dimensional Wasserstein--2 distance between the pushforward of the validation source sample and an independent validation sample from \(\nu\). In one dimension this quantity is computed exactly by sorting the two sample sets and taking the root mean square of the paired differences. Panel (a) of \cref{fig:numerical-rates} plots the mean validation \(W_2\) over the \(30\) runs as a function of \(N\). A least-squares fit in log-log coordinates gives $\log_{10}(W_2)\approx -0.5151\,\log_{10}(N)-0.0992.$ The fitted slope is close to \(-\tfrac12\), which is consistent with the \(N^{-1/2}\) scaling predicted by \cref{ex:1d-model}.

\paragraph{Two-dimensional experiment}
We use the source and target measures from \cref{ex:2d-model}. We again keep the architecture fixed and use a multilayer perceptron of architecture \([2,256,256,2]\). The other settings are the same as the one-dimensional problem. Source samples are drawn uniformly from \(B_1(0)\), while target samples are generated from \(\nu\).

For evaluation, we use an independent validation set of size \(3\times 10^4\). For each run we retain the iterate with smallest training loss and compute the Wasserstein--2 distance between the pushforward of the validation source sample and an independent validation sample from \(\nu\). Since no exact sorting formula is available in two dimensions, this validation \(W_2\) is computed numerically from the two empirical measures on the validation point clouds. Panel (b) of \cref{fig:numerical-rates} plots the mean validation \(W_2\) over the \(30\) runs as a function of \(N\). A least-squares fit in log-log coordinates gives $\log_{10}(W_2)\approx -0.2969\,\log_{10}(N)-0.1343.$
Thus the observed decay is slightly faster than the rate \(N^{-1/4}\) guaranteed by \cref{cor:fixed-arch-rate}, and is therefore consistent with the theory. This also indicates that, for the present two-dimensional example, the theoretical upper bound is not sharp.


\section{Conclusion}\label{sec:conclusion}

In this paper, we developed a theoretical framework for one-step
Wasserstein-guided generative modeling of PDE-induced probability measures. The
central theoretical result is \cref{thm:main1}, which establishes the
H\"older continuity of the population-optimal transport map for target measures
arising from several classes of PDE models. This result provides a regularity
foundation for approximating PDE-induced probability measures through a single
transport map from a reference distribution. Its proof relies on two main
ingredients: the doubling property for target measures generated by
bounded-domain elliptic and parabolic equations, and by Fokker--Planck dynamics
on the torus, as proved in
\cref{thm:elliptic_doubling,thm:parabolic_doubling,thm:torus_doubling}; and the
optimal-transport regularity result in \cref{lemma:caffarelli}, which together
yield the desired H\"older regularity.

We then analyzed DeepParticle as a representative one-step
Wasserstein-guided model. The excess-risk bound is given in
\cref{thm:main-excess-risk}; for fixed sufficiently expressive architectures
and sufficiently small optimization error, the corresponding sample-size rates
are summarized in \cref{cor:fixed-arch-rate}. We also obtained an
out-of-distribution risk bound in \cref{prop:ood-bound}, showing how the learned
transport map behaves when evaluated beyond the training distribution. The
numerical examples in \cref{ex:1d-model,ex:2d-model} and the experiments in
\cref{fig:numerical-rates} further support the predicted convergence behavior
and illustrate the practical relevance of the theoretical rates.

Overall, these results connect PDE regularity, optimal transport theory, and
statistical learning guarantees for one-step generative modeling. Future work
includes sharpening the dimension dependence in the statistical terms,
extending the theory to nonlinear or unbounded-domain PDE settings, and
developing optimization guarantees that better reflect practical training
dynamics.

\section*{Acknowledgements}
 ZZ was supported by the National Natural Science Foundation of China (Projects 92470103), the Hong Kong RGC grant (Projects 17304324 and 17300325), the Seed Funding Programme for Basic Research (HKU), and the Hong Kong RGC Research Fellow Scheme 2025. ZW was partly supported by NTU SUG-023162-00001, MOE AcRF Tier 1 Grant RG17/24. JX was partly supported by NSF grants DMS-2219904 and DMS-2309520. The computations were performed at the research computing facilities provided by the Information Technology Services, the University of Hong Kong, and the Greenplanet Cluster at UC Irvine.
\bibliographystyle{plain}
\bibliography{sample}

\end{document}